\documentclass{article}
\usepackage{ijcai16}

\usepackage{times}

\usepackage{amssymb,amsmath,amsthm}
\usepackage{xspace}
\usepackage{amsthm}
\usepackage{amsmath}
\usepackage{amssymb}
\usepackage{mathrsfs}
\usepackage{url}

\usepackage{algorithm,algorithmicx,algpseudocode}

\newcommand{\possent}{\vdash_{poss}}

\newtheorem{theorem}{Theorem}
\newtheorem{definition}{Definition}

\newtheorem{lemma}{Lemma}
\newtheorem{proposition}{Proposition}

\newtheorem{example}{Example}

\usepackage{stmaryrd}

\newcommand{\snake}{{\,\mid\!\sim\,}}

\title{Learning Possibilistic Logic Theories from Default Rules}

\author{Ond\v{r}ej Ku\v{z}elka \\ 
Cardiff University, UK  \\
KuzelkaO@cardiff.ac.uk \And 
Jesse Davis \\ 
KU Leuven, Belgium  \\
jesse.davis@cs.kuleuven.be \And
Steven Schockaert \\ 
Cardiff University, UK  \\
SchockaertS1@cardiff.ac.uk}

\newcommand*{\ARXIVVERSION}{}%

\begin{document}

\maketitle

\begin{abstract}
We introduce a setting for learning possibilistic logic theories from defaults of the form ``if alpha then typically beta''. We first analyse this problem from the point of view of machine learning theory, determining the VC dimension of possibilistic stratifications as well as the complexity of the associated learning problems, after which we present a heuristic learning algorithm that can easily scale to thousands of defaults. An important property of our approach is that it is inherently able to handle noisy and conflicting sets of defaults. Among others, this allows us to learn possibilistic logic theories from crowdsourced data and to approximate propositional Markov logic networks using heuristic MAP solvers. We present experimental results that demonstrate the effectiveness of this approach.
\end{abstract}

\section{Introduction}


Structured information plays an increasingly important role in applications such as information extraction \cite{dong2014knowledge}, question answering \cite{kalyanpur2012structured} and robotics \cite{beetz2011robotic}. With the notable exceptions of CYC and WordNet, most of the knowledge bases that are used in such applications have at least partially been obtained using some form of crowdsourcing (e.g.\ Freebase, Wikidata, ConceptNet). To date, such knowledge bases are mostly limited to facts (e.g.\ Obama is the current president of the US) and simple taxonomic relationships (e.g.\ every president is a human). 
One of the main barriers to crowdsourcing more complex domain theories is that most users are not trained in logic. This is exacerbated by the fact that often (commonsense) domain knowledge is easiest to formalize as defaults (e.g.\ birds typically fly), and, even for non-monotonic reasoning (NMR) experts, it can be challenging to formulate sets of default rules without introducing inconsistencies (w.r.t.\ a given NMR semantics) or unintended consequences.

In this paper, we propose a method for learning consistent domain theories from crowdsourced examples of defaults and non-defaults. Since these examples are provided by different users, who may only have an intuitive understanding of the semantics of defaults, together they will typically be inconsistent. The problem we consider is to construct a set of defaults which is consistent w.r.t.\ the System P semantics \cite{KLM}, and which entails as many of the given defaults and as few of the non-defaults as possible. Taking advantage of the relation between System P and possibilistic logic \cite{benferhat1997nonmonotonic}, we treat this as a learning problem, in which we need to select and stratify a set of propositional formulas.
 
The contributions of this paper are as follows. First, we show that the problem of deciding whether a possibilistic logic theory exists that perfectly covers all positive and negative examples is $\Sigma_2^P$-complete. Second, we formally study the problem of learning from defaults in a standard learning theory setting and we determine the corresponding VC-dimension, which allows us to derive theoretical bounds on how much training data we need, on average, to obtain a system that can classify defaults as being valid or invalid with a given accuracy level. Third, we introduce a heuristic  algorithm for learning possibilistic logic theories from defaults and non-defaults. To the best of our knowledge, our method is the first that can learn a consistent logical theory from a set of noisy defaults. We evaluate the performance of this algorithm in two crowdsourcing experiments. In addition, we show how it can be used for approximating maximum a posteriori (MAP) inference in propositional Markov logic networks \cite{Richardson2006,Saint-cyr94penaltylogic}.

\section{Related work}
Reasoning with defaults of the form ``if $\alpha$ then typically $\beta$'', denoted as $\alpha \snake \beta$, has been widely studied \cite{KLM,pearl1990system,lehmann1992does,geffner1992conditional,goldszmidt1993maximum,benferhat1997nonmonotonic}. A central problem in this context is to determine what other defaults can be derived from a given input set. Note, however, that the existing approaches for reasoning about default rules all require some form of consistency (e.g.\ the input set cannot contain both $a\snake b$ and $a\snake \neg b$). As a result, these approaches cannot directly be used for reasoning about noisy crowdsourced defaults.

To the best of our knowledge, this is the first paper that considers a machine learning setting where the input consists of default rules. Several authors have proposed approaches for constructing possibility distributions from data; see \cite{dubois_prade_survey} for a recent survey. However, such methods are generally not practical for constructing possibilistic logic theories. The possibilistic counterpart of the Z-ranking constructs a possibilistic logic theory from a set of defaults, but it requires that these defaults are consistent and cannot handle non-defaults \cite{benferhat1997nonmonotonic}, although an extension of the Z-ranking that can cope with non-defaults was proposed in \cite{booth1998note}. Some authors have also looked at the problem of learning sets of defaults from data \cite{big_stepped,kernDefaultLearning}, but the performance of these methods has not been experimentally tested. In \cite{Serrurier2007939}, a possibilistic inductive logic programming (ILP) system is proposed, which uses a variant of possibilistic logic for learning rules with exceptions. However, as is common for ILP systems, this method only considers classification problems, and cannot readily be applied to learn general possibilistic logic theories. Finally note that the setting of learning from default rules as introduced in this paper can be seen as a non-monotonic counterpart of an ILP setting called {\em learning from entailment} \cite{de1997logical}.

\section{Background}

\subsection{Possibilistic logic}
A stratification of a propositional theory $\mathcal{T}$ is an ordered partition of $\mathcal{T}$. We will use the notation $\textit{Strat}(\mathcal{T})$ to denote the set of all such ordered partitions and $\textit{Strat}^{(k)}(\mathcal{T})$ to denote the set of all ordered partitions into at most $k$ subsets of $\mathcal{T}$. A theory in possibilistic logic \cite{DLP} is a set of formulas of the form $(\alpha,\lambda)$, with $\alpha$ a propositional formula and $\lambda \in ]0,1]$ a certainty weight. These certainty weights are interpreted in a purely ordinal fashion, hence a possibilistic logic theory is essentially a stratification of a propositional theory.
The strict $\lambda$-cut $\Theta_{\overline{\lambda}}$ of a possibilistic logic theory $\Theta$ is defined as 
$\Theta_{\overline{\lambda}} = \{\alpha \,|\, (\alpha,\mu)\in \Theta, \mu > \lambda\}$. The inconsistency level $\textit{inc}(\Theta)$ of $\Theta$ is the lowest certainty level $\lambda$ in $[0,1]$ for which the classical theory $\Theta_{\overline{\lambda}}$ is consistent.
An inconsistency-tolerant inference relation $\possent$ for possibilistic logic can then be defined as follows:
\begin{align*}
\Theta \possent \alpha \quad\text{iff}\quad \Theta_{\overline{\textit{inc}(\Theta)}} \models \alpha
\end{align*}
We will write $(\Theta,\alpha)\possent \beta$ as an abbreviation for $\Theta\cup \{(\alpha,1)\} \possent \beta$. 
It can be shown that $\Theta \possent (\alpha,\lambda)$ can be decided by making $O(\log_2 k)$ calls to a SAT solver, with $k$ the number of certaintly levels in $\Theta$ \cite{La2001.1}. 

There is a close relationship between possibilistic logic and the rational closure of a set of defaults. 
Recall that $\alpha \snake \beta$ is tolerated by a set of defaults $\{\alpha_1\snake \beta_1,...,\alpha_n\snake \beta_n\}$ if the classical formula $\alpha \wedge \beta \wedge \bigwedge_i (\neg \alpha_i \vee \beta_i)$ is consistent \cite{pearl1990system}. Let $\Delta$ be a set of defaults. The rational closure of $\Delta$ is based on a stratification $\Delta_1,...,\Delta_k$, known as the Z-ordering, where each $\Delta_j$ contains all defaults from $\Delta\setminus (\Delta_{1}\cup ... \Delta_{j-1})$ which are tolerated by $\Delta\setminus (\Delta_1\cup ... \cup \Delta_{j-1})$. Intuitively, $\Delta_1$ contains the most general default rules, $\Delta_2$ contains exceptions to these rules, $\Delta_3$ contains exceptions to these exceptions, etc. Given the stratification $\Delta_1,...,\Delta_k$ we define the possibilistic logic theory $\Theta = \{(\neg \alpha \vee \beta, \lambda_i) \,|\, (\alpha\snake\beta)\in \Delta_i\}$, where we assume $0<\lambda_1<...<\lambda_k \leq 1$. It then holds that $\alpha \snake \beta$ is in the rational closure of $\Delta$ iff $(\Theta,\alpha)\possent \beta$ \cite{benferhat1998practical}.

\subsection{Learning Theory}

We now cover some basic notions from statistical learning theory \cite{Vapnik:1995:NSL:211359}. We restrict ourselves to binary classification problems, where the two labels are $1$ and $-1$. Let $\mathcal{X}$ be a set of {\em examples}. A {\em hypothesis} is a function $h : \mathcal{X} \rightarrow \{-1,1\}$. A hypothesis $h$ is said to cover an example $e \in \mathcal{X}$ if $h(e) = 1$. 
%
Consider a set $\mathcal{S}\subseteq \mathcal{X} \times \{-1,1\}$ of $n$ labeled examples that have been iid sampled from a distribution $p$.
A hypothesis $h$'s sample error rate is
$
\textit{err}_{\mathcal{S}}(h,\mathcal{S}) = \frac{1}{n} \sum_{(x,c) \in \mathcal{S}} \mathbf{1}(h(x) \neq c)
$
where $\mathbf{1}(\alpha)=1$ if $\alpha\equiv\textit{true}$ and $\mathbf{1}(\alpha)=0$ otherwise.
A hypothesis $h$'s expected error w.r.t.\ the probability distribution $p$ is given by
$
\textit{err}_p(h) = \mathbf{E}_{(X,C) \sim p} [\mathbf{1}(h(X) \neq C)].
$
Statistical learning theory provides tools for bounding the probability $P(\sup_{h \in \mathcal{H}}|\textit{err}_p(h)-\textit{err}_\mathcal{S}(h,\mathcal{S})| \geq \epsilon)$, where $\mathcal{S}$ is known to be sampled iid from $p$ but $p$ itself is unknown. 
These bounds link $h$'s training set error to its (probable) performance on other examples drawn from the same distribution, and therefore permits theoretically controlling overfitting. The most important bounds of this type depend on the Vapnik-Chervonenkis (VC) dimension \cite{Vapnik:1995:NSL:211359}.

\begin{definition}[Vapnik-Chervonenkis (VC) dimension]
A hypothesis set $\mathcal{H}$ is said to shatter a set of examples $\mathcal{Y}$ if for every subset $\mathcal{Z} \subseteq \mathcal{Y}$ there is a hypothesis $h \in \mathcal{H}$ such that $h(e) = 1$ for every $e \in \mathcal{Z}$ and $h(e) = -1$ for every $e \in \mathcal{Y} \setminus \mathcal{Z}$. The VC dimension of $\mathcal{H}$ is the cardinality of the largest set that is shattered by $\mathcal{H}$.
\end{definition}

\noindent Upper bounds based on the VC dimension are increasing functions of the VC dimension and decreasing functions of the number of examples in the training sample $\mathcal{S}$.
Ideally, the goal is to minimize expected error, but this cannot be evaluated since  $p$ is unknown.  {\em Structural risk minimization} \cite{Vapnik:1995:NSL:211359} helps with this  
if the hypothesis set can be organized into a hierarchy of nested hypothesis classes of increasing VC dimension. It suggests selecting hypotheses that minimize a risk composed of the training set error and a complexity term, e.g.\ if two hypotheses have the same training set error, the one originating from the class with lower VC dimension should be preferred.

\section{Learning from Default Rules}

In this section, we formally describe a new learning setting for possibilistic logic called {\em learning from default rules}. 
We assume a finite alphabet $\Sigma$ is given. An example is a default rule over $\Sigma$ and a hypothesis is a possibilistic logic theory over $\Sigma$. A hypothesis $h$ predicts the class of an example $e = \alpha \snake \beta$ by checking if $h$ covers $e$, in the following sense.
\begin{definition}[Covering]
A hypothesis $h \in \mathcal{H}$  {\em covers} an example $e = \alpha \snake \beta$ if $(h, \alpha) \possent \beta$.
\end{definition}
\noindent The hypothesis $h$ predicts positive, i.e.\ $h(\alpha \snake \beta) = 1$, iff $h$ covers $e$, and else predicts negative, i.e.\ $h(\alpha \snake \beta) = -1$.

\begin{example}
Let us consider the following set of examples
\begin{align*}
\mathcal{S} =& \{ (\textit{bird} \wedge \textit{antarctic} \snake \neg \textit{flies}, 1), (\textit{bird} \snake \neg \textit{flies}, -1) \}
\end{align*}
The following hypotheses over the alphabet $\{\textit{bird},\allowbreak \textit{flies},\allowbreak \textit{antarctic} \}$ cover all positive and no negative examples: 
\begin{align*}
h_1 &= \{ (\textit{bird}, 1), (\textit{antarctic} \rightarrow \neg \textit{flies}, 1)  \}\\
h_2 &= \{ (\textit{flies}, 0.5), (\textit{antarctic} \rightarrow \neg \textit{flies}, 1)  \}\\
h_3 &= \{(\textit{antarctic} \rightarrow \neg \textit{flies}, 1)  \}
\end{align*}
\end{example}
\noindent The learning task can be formally described as follows.
\begin{description} 
\item[Given:]  A multi-set $\mathcal{S}$ which is an iid sample from a set of default rules over a given finite alphabet $\Sigma$. 
\item[Do:] Learn a possiblistic logic theory that covers all positive examples and none of the negative examples in $\mathcal{S}$. 
\end{description}
The above definition assumes that $\mathcal{S}$ is perfectly separable, i.e.\ it is possible to perfectly distinguish positive examples from negative examples. In practice, we often relax this requirement, and instead aim to find a theory that minimizes the training set error. 
Similar to learning in graphical models, this learning task can be decomposed into {\em parameter learning} and {\em structure learning}. In our context, the goal of parameter learning is to convert a set of propositional formulas into a possibilistic logic theory, while the goal of structure learning is to decide what that set of propositional formulas should be.

\subsection{Parameter Learning}

Parameter learning assumes that the formulas of the possibilistic logic theory are fixed, and only the certainty weights need to be assigned. As the exact numerical values of the certainty weights are irrelevant, we will treat parameter learning as the process of finding the most suitable stratification of a given set of formulas, e.g.\ the one which minimizes training error or structural risk (cf.\ Section \ref{sec:vc}).

\begin{example}\label{ex:learning1}
Let 
$\mathcal{S} = \allowbreak\{ (\textit{penguin} \snake \textit{bird}, 1),\allowbreak (\textit{bird} \snake \textit{flies}, 1),\allowbreak (\textit{penguin} \snake \allowbreak \neg \textit{flies},\allowbreak 1),\allowbreak (\snake \textit{bird}, -1),\allowbreak  (\textit{bird} \snake\allowbreak \textit{penguin},\allowbreak -1) \}$
and  $\mathcal{T} = \{ \textit{bird},\allowbreak \textit{flies},\allowbreak \textit{penguin},\allowbreak \neg \textit{penguin} \allowbreak\vee \neg \textit{flies} \}.$
A stratification of $\mathcal{T}$ which minimizes the training error on the examples from $\mathcal{S}$ is 
$
\mathcal{T}^* = \{ (\textit{bird},\allowbreak 0.25),\allowbreak (\textit{penguin},\allowbreak 0.25),\allowbreak (\textit{flies},\allowbreak 0.5),\allowbreak  (\neg \textit{penguin} \vee \neg \textit{flies},\allowbreak 1) \}
$
which is equivalent to 
$\mathcal{T}^{**} = \{(\textit{flies},\allowbreak 0.5),\allowbreak  (\neg \textit{penguin} \vee \neg \textit{flies},\allowbreak 1) \}$ because $\textit{inc}(\mathcal{T}^*) = 0.25$. Note that $\mathcal{T}^{**}$ correctly classifies all examples except $(\textit{penguin} \snake \textit{bird}, 1)$.
\end{example}
\noindent Given a set of examples $\mathcal{S}$, we write $\mathcal{S}^+=\{\alpha |  (\alpha,1) \in \mathcal{S} \}$ and $\mathcal{S}^- = \{ \alpha |  (\alpha,-1) \in \mathcal{S} \}$). A stratification $\mathcal{T}^*$ of a theory $\mathcal{T}$ is a {\em separating stratification} of $\mathcal{S}^+$ and $\mathcal{S}^-$ if it covers all examples from $\mathcal{S}^+$ and no examples from $\mathcal{S}^-$. 

\begin{example}\label{exZrankingComparison}
Let us consider the following set of examples
$
\mathcal{S} = \{ (\snake \neg x, 1), (\snake \neg y, 1), (x \snake a, 1), (y \snake b, 1),
 (x \wedge y \snake a, -1) \}.
$
Let $\mathcal{T} = \{ \neg x, \neg y, \neg x \vee a, \neg y \vee b \}$.
The following stratification is a separating stratification of $\mathcal{S}^+$ and $\mathcal{S}^-$:
$h = \{ (\neg x, 0.25), (\neg x \vee a, 0.5), (\neg y, 0.75), (\neg y \vee b, 1) \}$. Note that the Z-ranking of $\mathcal{S}^+$ also corresponds to a stratification of $\mathcal{T}$, as $\mathcal{T}$ contains exactly the clause representations of the positive examples. However using the Z-ranking leads to a different stratification, which is:
$h_z = \{ (\neg x, 0.5), (\neg y, 0.5), (\neg x \vee a, 1), (\neg y \vee b, 1) \}.$
Note that $h_z(x \wedge y \snake a) = 1$ whereas $h(x \wedge y \snake a) = -1$.
\end{example}


\noindent Because arbitrary stratifications can be chosen, there is substantial freedom to ensure that negative examples are not covered. This is true even when the set of considered formulas is restricted to the clause representations of the positive examples, as seen in Example \ref{exZrankingComparison}. Unfortunately, the problem of finding an optimal stratification is computationally hard.


\begin{theorem}\label{thm-complexity}
Deciding whether a separating stratification exists for given $\mathcal{T}$, $\mathcal{S}^+$ and $\mathcal{S}^-$ is a $\Sigma_2^P$-complete problem.
\end{theorem}
\begin{proof}
The proof of the membership result is trivial. 
We show the hardness result by reduction from the $\Sigma_2^P$-complete problem of deciding the satisfiability of quantified Boolean formulas of the form $\exists X \forall Y : \Phi(X,Y)$ where $X$ and $Y$ are vectors of propositional variables and $\Phi(X,Y)$ is a propositional formula. Let 
$\mathcal{T} = X \cup \{ \neg x : x \in X \} \cup \{ \Phi(X,Y) \rightarrow \textit{aux} \}$
be a propositional theory, let $\mathcal{S}^+ = \{ \snake \textit{aux} \}$ and $\mathcal{S}^- = \emptyset$. We need to show that $\exists X \forall Y : \Phi(X,Y)$ is satisfiable if and only if there exists a separating stratification for $\mathcal{T}$, $\mathcal{S}^+$ and $\mathcal{S}^-$. ($\Rightarrow$) Let $\theta$ be an assignment of variables in $X$ such that $\forall Y : \Phi(X\theta,Y)$ is true. Then we can construct the separating stratification as 
\begin{align*}
(\{\Phi(X,Y) \rightarrow \textit{aux} \} \cup \{ x \in X : x\theta = 1 \} \\
\cup \{\neg x : x \in X \mbox{ and } x\theta = 0 \}, \\
\{ \neg x : x \in X \mbox{ and } x\theta = 1 \} \cup \{ x \in X : x\theta = 0 \}).
\end{align*}
Since $\Phi(X,Y)$ will always be true in any model consistent with the highest level of the stratification, because of the way we chose $x$ and $\neg x$ for this level, so will $\textit{aux}$.
($\Leftarrow$) Let $\mathcal{T}^*$ be a stratification of $\mathcal{T}$ which entails the default rule $\snake \textit{aux}$. We can assume w.l.o.g. that $\mathcal{T}^*$ has only two levels. Since $\mathcal{T}^*$ is a separating stratification, we must have $\mathcal{T}^* \possent \textit{aux}$. Therefore the highest level $L^*$ of $\mathcal{T}^*$ must be a consistent theory and $\Phi(X,Y)$ must be true in all of its models.
Let $X' = \{ x \in X : x \in L^* \mbox{ or } \neg x \in L^* \}$ and $X'' = X\setminus X'$. We can construct an assignment $\theta$ to variables in $X'$ by setting $x\theta = 1$ for $x \in L^*$ and $x\theta = 0$ for $\neg x \in L^*$. It follows from the construction that $\forall X'' \forall Y : \Phi(X\theta,Y)$ must be true.
\end{proof}
\noindent As this result reveals, in practice we will need to rely on heuristic methods for parameter learning. In Section \ref{secHeuristicAlgorithm} we will propose such a heuristic method, which will moreover also include structure learning.

\subsection{VC Dimension of Possibilistic Logic Theories}\label{sec:vc}

We explore the VC dimension of the set of possible stratifications of a propositional theory, as this will allow us to provide probabilistic bounds on the generalization ability of a learned possibilistic logic theory.
Let us write $\textit{Strat}(\mathcal{T})$ for the set of all stratifications of a propositional theory $\mathcal{T}$, and let $\textit{Strat}^{(k)}(\mathcal{T})$ be the set of all stratifications with at most $k$ levels. The following proposition provides an upper bound for the VC dimension and can be proved by bounding the cardinality of $\textit{Strat}^{(k)}(\mathcal{T})$.

\begin{proposition}\label{prop:upperbound}
Let $\mathcal{T}$ be a set of $n$ propositional formulas. Then $VC(\textit{Strat}^{(k)}(\mathcal{T})) \leq n \log_2{k}$. 
\end{proposition}

\noindent In the next theorem, we establish a lower bound on the VC dimension of stratifications with at most $k$ levels which shows that the above upper bound is asymptotically tight. 

\begin{theorem}\label{thm-vc2}
For every $k, n$, $k \leq n$, there is a propositional theory $\mathcal{T}$ consisting of $n$ formulas such that $$VC(\textit{Strat}^{(k)}(\mathcal{T})) \geq \frac{1}{4} n (\log_2{k}-1).$$
\end{theorem}

\noindent To prove Theorem \ref{thm-vc2}, we need the following lemmas; some straightforward proofs are omitted due to space constraints.

\begin{lemma}\label{lemma:orders}
If $\mathcal{S}$ is a totally ordered set, let $\textit{kth}(\mathcal{S},i)$ denote the $i$-th highest element of $\mathcal{S}$.
Let $\mathcal{X} = \{x_1,\dots,x_n\}$ be a set of cardinality $n = 2^m$ where $m \in \mathbb{N}\setminus \{ 0 \}$. Let 
\begin{align*}
\mathcal{C} &= \{ x_1 < x_2, x_3 < x_4, \dots, x_{n-1} < x_{n}, \\
&\textit{kth}(\{x_1,x_2\}, 1) < \textit{kth}(\{x_3,x_4\}, 1), \\
&\textit{kth}(\{x_1,x_2\}, 2) < \textit{kth}(\{x_3,x_4\}, 2), \\
&\textit{kth}(\{x_5,x_6\}, 1) < \textit{kth}(\{x_7,x_8\}, 1), \\
&\textit{kth}(\{x_5,x_6\}, 2) < \textit{kth}(\{x_7,x_8\}, 2), \\
&\dots \\
&\textit{kth}(\{x_1,x_2,x_3,x_4\}, 1) < \textit{kth}(\{x_5,x_6,x_7,x_8\}, 1), \\
&\textit{kth}(\{x_1,x_2,x_3,x_4\}, 2) < \textit{kth}(\{x_5,x_6,x_7,x_8\}, 2), \\
&\textit{kth}(\{x_1,x_2,x_3,x_4\}, 3) < \textit{kth}(\{x_5,x_6,x_7,x_8\}, 3), \\
&\textit{kth}(\{x_1,x_2,x_3,x_4\}, 4) < \textit{kth}(\{x_5,x_6,x_7,x_8\}, 4), \\
&\dots \\
& \textit{kth}(\{x_1,\dots,x_{n/2}\}, n/2) < \textit{kth}(\{x_{n/2+1},\dots,x_{n}\}, n/2) \}
\end{align*}
be a set of $\frac{1}{2} n \log_2{n}$ inequalities. Then for any $\mathcal{C}' \subseteq \mathcal{C}$ there is a permutation of $\mathcal{X}$ satisfying all constraints from $\mathcal{C}'$ and no constraints from $\mathcal{C}\setminus \mathcal{C}'$.
\end{lemma}

\begin{lemma}\label{lemma:pos_order}
Let $\textit{at-least}_{k}(x_1,x_2,\dots,x_n)$ denote a Boolean formula which is true if and only if at least $k$ of the arguments are true. Let $\mathcal{X} = \{x_1,\dots,x_m\}$ be a set of propositional logic variables and $\pi = (x_{i_1}, x_{i_2}, \dots, x_{i_m})$ be a permutation of elements from $\mathcal{X}$. Let $0 \leq k \leq \min\{m_y,m_z\}$. Let $$\mathcal{T}^* = \{(x_{i_1}, 1/m), (x_{i_2},1/(m-1)), \dots, (x_{i_m}, 1)\}$$ be a possibilistic logic theory. Let $\mathcal{Y} = \{y_1,\dots,y_{m_y}\}$ and $\mathcal{Z} = \{z_1,\dots,z_{m_z} \}$ be disjoint subsets of $\mathcal{X}$. Then 
\begin{align*}
(\mathcal{T}^*, \textit{at-least}_{m_y-k+1}(\neg y_1,\neg y_2,\dots,\neg y_{m_y})) \possent \\
\textit{at-least}_{k}(z_1,z_2,\dots,z_{m_z})
\end{align*}
iff $\textit{kth}(\{y_1,\dots,y_{m_y}\}, k) < \textit{kth}(\{z_1,\dots,z_{m_z}\},k)$ 
w.r.t.\ the ordering given by the permutation $\pi$.
\end{lemma}

\begin{lemma}\label{lemma:vc3}
For every $n = 2^m$ there is a propositional theory $\mathcal{T}$ consisting of $n$ formulas such that $$VC(\textit{Strat}(\mathcal{T})) \geq \frac{1}{2} n \log_2{n}.$$
\end{lemma}
\begin{proof}
Let $\mathcal{T} = \{x_1,\dots,x_n\}$ be a set of propositional variables where $n = 2^m$, $m \in N\setminus \{ 0 \}$ and let $\mathcal{C}$ be defined as in Lemma \ref{lemma:orders}. Let $\mathcal{D} = $
\begin{align*}
&\{ \textit{at-least}_{l-k+1}(\neg x_{i_1},\dots,\neg x_{i_l}) \snake \textit{at-least}_{k}(x_{j_1},\dots,x_{j_l}) | \\
&\quad\quad(\textit{kth}(\{x_{i_1},\dots,x_{i_l}\}, k) < \textit{kth}(\{x_{j_1},\dots,x_{j_l}\}, k)) \in \mathcal{C}\},
\end{align*}
i.e.\ $\mathcal{D}$ contains one default rule for every inequality from $\mathcal{C}$. It follows from Lemma \ref{lemma:orders} and Lemma \ref{lemma:pos_order} that the set $\mathcal{D}$ can be shattered by stratifications of the propositional theory $\mathcal{T}$. The cardinality of $\mathcal{D}$ is $\frac{1}{2} n \log_2 n$. Therefore the VC dimension of stratifications of $\mathcal{T}$ is at least $\frac{1}{2} n \log_2 n$.
\end{proof}

\begin{proof}[Proof of Theorem \ref{thm-vc2}]
We show that if $k$ and $n$ are powers of two then $\frac{1}{2} n \log_2 k$ is a lower bound of the VC dimension. The general case of the theorem then follows straightforwardly. Let $\mathcal{T} = \bigcup_{i=1}^{\frac{n}{k}}\{ x_{(i-1) \cdot k + 1},\allowbreak\dots,\allowbreak x_{i \cdot k} \}$ and let $\mathcal{D}_i$ be a set of default rules of cardinality $\frac{1}{2} k \log_2 k$ shattered by $\textit{Strat}(\{ x_{(i-1) \cdot k + 1},\dots, x_{i \cdot k} \})$. It follows from Lemma \ref{lemma:vc3} that such a set $\mathcal{D}_i$ always exists. Let $\mathcal{D} = \bigcup_{i=1}^{\frac{n}{k}} \mathcal{D}_i$. Then $\mathcal{D}$ has cardinality $\frac{1}{2} n \log_2 k$ and is shattered by $\textit{Strat}^{(k)}(\mathcal{T})$. To see that the latter holds, note that the sets of formulas $\textit{Strat}(\{ x_{(i-1) \cdot k + 1},\dots, x_{i \cdot k} \})$ are disjoint. Therefore, if we want to find a stratification from $\textit{Strat}^{(k)}(\mathcal{T})$ which covers only examples from an arbitrary set $\mathcal{D}' \subseteq \mathcal{D}$ and no other examples from $\mathcal{D}$ then we can merge stratifications of $\{ x_{(i-1) \cdot k + 1},\dots, x_{i \cdot k} \}$ which cover exactly the examples from $\mathcal{D}_i \cap \mathcal{D}'$, where merging stratifications is done by level-wise unions.
\end{proof}

\noindent Combining the derived lower bounds and upper bounds on the VC dimension together with the structural risk minimization principle, we find that given two stratifications with the same training set error rate, we should prefer the one with the fewest levels. 
Furthermore, when structure learning is used, it is desirable for learned theories to be compact. A natural learning problem then consists in selecting a small subset of $\mathcal{T}$, where $\mathcal{T}$ corresponds to the set of formulas considered by the structure learner, and identifying a stratification only for that subset. The results in this section can readily be extended to provide bounds on the VC dimension of this problem. Let $\mathcal{T}$ be a propositional theory of cardinality $n$ and let $m < n$ be a positive integer. The VC dimension of the set of hypotheses involving at most $m$ formulas from $\mathcal{T}$ and having at most $k$ levels is bounded by $m (\log_2 n + \log_2 k)$. This can simply be obtained by upper-bounding the number of the different stratifications with at most $k$ levels and $m$ formulas selected from a set of cardinality $n$, by $n^m \cdot k^m$.

\subsection{Heuristic Learning Algorithm}\label{secHeuristicAlgorithm}

In this section, we propose a practical heuristic algorithm for learning a possibilistic logic theory from a set $\mathcal{S}$ of positive and negative examples of default rules. Our method combines greedy structure learning with greedy weight learning. We assume that every default or non-default $\alpha \snake \beta$ in $\mathcal{S}$ is such that $\neg \alpha$ and $\beta$ correspond to clauses.

The algorithm starts by initializing the ``working'' stratification $\mathcal{T}^*$ to be an empty list. Then it repeats the following revision procedure for a user-defined number of iterations $n$, or until a timeout is reached. First, it generates a set of candidate propositional clauses $C$ as follows:
\begin{itemize}
    \item It samples a set of defaults $\alpha \snake \beta$ from the examples that are misclassified by $\mathcal{T}^*$.
    \item For each default $\alpha \snake \beta$ which has been sampled, it samples a subclause $\neg\alpha'$ of $\neg\alpha$ and a subclause $\beta'$ of $\beta$. If $\alpha \snake \beta$ is a positive example then $\neg \alpha' \vee \beta'$ is added to $C$; if it is a negative example, then $\neg \alpha' \vee \beta''$ is added instead, where $\beta''$ is obtained from $\beta'$ by negating each of the literals.
\end{itemize}
The algorithm then tries to add each formula in $C$ to an existing level of $\mathcal{T}^*$ or to a newly inserted level. It picks the clause $c$ whose addition leads to the highest accuracy and adds it to $\mathcal{T}^*$. The other clauses from $C$ are discarded. In case of ties, the clause which leads to the  stratification with the fewest levels is selected, in accordance with the structural risk minimization principle and our derived VC dimension. If there are multiple such clauses, then it selects the shortest among them. Subsequently, the algorithm tries to greedily minimize the newly added clause $c$, by repeatedly removing literals as long as this does not lead to an increase in the training set error. Next, the algorithm tries to revise $\mathcal{T}^*$ by greedily removing clauses whose deletion does not increase the training set error. Finally, as the last step of each iteration, the weights of all clauses are optimized by greedily reinserting each clause in the theory.


\section{Experiments}
We evaluate our heuristic learning algorithm\footnote{\label{footnote-online}The data, code, and learned models are available from \url{https://github.com/supertweety/}.} in two different applications: learning domain theories from crowdsourced default rules and approximating MAP inference in propositional Markov logic networks. As we are not aware of any existing methods that can learn a consistent logical theory from a set of noisy defaults, there are no baseline methods to which our method can directly be compared. However, if we fix a target literal $l$, we can train standard classifiers to predict for each propositional context $\alpha$ whether the default $\alpha \snake l$ holds. This can only be done consistently with ``parallel'' rules, where the literals in the consequent do not appear in antecedents. We will thus compare our method to three traditional classifiers on two crowdsourced datasets of parallel rules: random forests \cite{breiman.rf}, C4.5 decision trees \cite{quinlan.c4.5}, and the rule learner RIPPER \cite{ripper}. Random forests achieve state-of-the-art accuracy\footnote{A recent large-scale empirical evaluation has shown that variants of the random forest algorithm tend to perform best on real-life datasets \cite{JMLR:v15:delgado14a}.} but its models are difficult to interpret. Decision trees are often less accurate but more interpretable than random forests. Finally, rule learners have the most interpretable models, but often at the expense of lower accuracy. In the second experiment, approximating MAP inference, we do not restrict ourselves to parallel rules. In this case, only our method can guarantee that the predicted defaults will be consistent. 

\subsection{Methodology}

Our learning algorithm is implemented in Java and uses the SAT4j library \cite{sat4j}. The implementation contains a number of optimizations which make it possible to handle datasets of thousands of default rules, including caching, parallelization, detection of relevant possibilistic subtheories for deciding entailment queries and unit propagation in the possibilistic logic theories. 

We use the Weka \cite{weka} implementations for the three baselines. 
When using our heuristic learning algorithm, we run it for a maximum time of 10 hours for the crowdsourcing experiments reported in Section \ref{sec:exp-crowd} and for one hour for the experiments reported in Section \ref{sec:exp-map}. For C4.5 and RIPPER, we use the default settings. For random forests, we used the default settings and set the number of trees to 100.

\subsection{Learning from Crowdsourced Examples}\label{sec:exp-crowd}

We used CrowdFlower, an online crowdsourcing platform, to collect expert rules about two domains. 
In the first experiment, we created 3706 scenarios for a team on offense in American football by varying the field position, down and distance, time left, and score difference. Then we presented six choices for a play call (punt, field goal, run, pass, kneel down, do not know/it depends) and asked the user to select the most appropriate one. All scenarios were presented to 5 annotators.
A manual inspection of a subset of the rules revealed that they are of reasonably high quality. In a second experiment, users were presented with 2388 scenarios based on Texas hold'em poker situations, where users were asked whether in a given situation they would typically fold, call or raise, with a fourth option again being ``do not know/it depends''. Each scenario was again presented to 5 annotators. Given the highly subjective nature of poker strategy, it was not possible to enforce the usual quality control mechanism on CrowdFlower in this case, and the quality of the collected rules was accordingly found to be more variable.

In both cases, the positive examples are the rules obtained via crowdsourcing, while negative examples are created by taking positive examples and randomly selecting a different consequent.
To create training and testing sets, we divided the data based on annotator ID so that all rules labeled by a given annotated appear only in the training set or only in the testing set, to prevent leakage of information. We added a set of hard rules to the possibilistic logic theories to enforce that only one choice should be selected for a game situation. The baseline methods were presented with the same information, in the sense that the problem was presented as a multi-class classification problem, i.e.\ given a game situation, the different algorithms were used to predict the most typical action (with one additional option being that none of the actions is typical). The results are summarized in Table \ref{tab:mrfs}.


In the poker experiment, our approach obtained slightly higher accuracy than random forest and RIPPER but performed slightly worse than C4.5. However, a manual inspection showed that a meaningful theory about poker strategy was learned.
For example, at the lowest level, the possibilistic logic theory contains the rule ``call'', which makes sense given the nature of the presented scenarios. At a higher level, it contains more specific rules such as ``if you have three of a kind then raise''. At the level above, it contains exceptions to these more specific rules such as ``If you have three of a kind, there are three hearts on the board and your opponent raised on the river then call''.

In the American football experiment, our approach obtained lower accuracy than the competing algorithms. The best accuracy was achieved by C4.5. Again, we also manually inspected the learned possibilistic logic theory and found that it captures some general intuitions and known strategy about the game. For example, the most general rule is "pass" which is the most common play type. Another example is that second most general level has several rules that say on fourth down and long you should punt. More specific levels that allow for cases when you should not punt, such as when you are in field goal range.


Despite not achieving the same accuracy as C4.5 in this experiment, it nonetheless seems that our method is useful for building up domain theories by crowdsourcing opinions. The learned domain theories are easy to interpret (e.g., the size of the poker theory, as a sum of rule lengths, is more than 10 times smaller than the number of nodes in the learned tree) and capture relevant strategies for both games. The models obtained by classifiers such as C4.5, on the other hand, are often difficult to interpret. Moreover, traditional classifiers such as C4.5 can only be applied to parallel rules, and will typically lead to inconsistent logical theories in more complex domains. In contrast, our method can cope with arbitrary default rules as input, making it much more broadly applicable for learning domain theories.

\begin{table}
\center
\begin{tabular}{ l | c c c c}
    & Poss. & Rand. F. & C4.5 & RIPPER \\ \hline
Poker & 40.5 & 38.6 & \bf 41.1 & 39.9 \\
Football & 68.3 & 72.4 & \bf 74.6 & 73.1 \\ \hline
NLTCS & \bf 78.1 & 69.6 & 70.2 & 67.7 \\
MSNBC & \bf 62.0 & 61.9 & \bf 62.0 & 48.8 \\
Plants & 73.1 & \bf 77.8 & 71.4 & 53.8 \\
DNA & 52.8 & \bf 56.6 & 54.9 & 51.1 \\ 
\end{tabular}\caption{Test set accuracies.}\label{tab:mrfs}
\end{table}

\subsection{Approximating MAP Inference}\label{sec:exp-map}
Markov logic networks can be seen as weighted logical theories. The weights assigned to formulas are intuitively seen as penalties; they are used to induce a probability distribution over the set of possible world. Here we are interested in maximum a posteriori (MAP) inference. Specifically, we consider the following entailment relation from \cite{Saint-cyr94penaltylogic}: 
$(\mathcal{M},\alpha) \vdash_{\textit{MAP}} \beta \quad\text{iff}\quad \forall \omega \in \max(\mathcal{M},\alpha): \omega \models \beta$
where $\mathcal{M}$ is an MLN, $\alpha$ and $\beta$ are propositional formulas and $\max(\mathcal{M},\alpha)$ is the set of most probable models of $\alpha$, w.r.t.\ the probability distribution induced by $\mathcal{M}$. Note that MAP inference only depends on an ordering of possible worlds. It was shown in \cite{mln2posl} that for every propositional MLN $\mathcal{M}$ there exists a possibilistic logic theory $\Theta$ such that $(\mathcal{M},\alpha) \vdash_{\textit{MAP}} \beta$ iff $(\Theta,\alpha) \possent \beta$. Such a translation can be useful in practice, as possibilistic logic theories tend to be much easier to interpret, given that the weights associated with different formulas in an MLN can interact in non-trivial ways. Unfortunately, in general $\Theta$ is exponentially larger than $\mathcal{M}$. Moreover, the translation from \cite{mln2posl} requires an exact MAP solver, whereas most such solvers are approximate.

Therefore, rather than trying to capture MAP inference exactly, here we propose to learn a possibilistic logic theory from a set of examples of valid and invalid MAP entailments $(\mathcal{M},\alpha) \vdash_{\textit{MAP}} \beta$. Since our learning algorithm can handle non-separable data, we can use approximate MAP solvers for generating these examples, which leads to further gains in scalability.
As is common, the evidence $\alpha$ consists of conjunctions of up to $k$ literals, and the conclusion $\beta$ consists of an individual literal. To create examples, we randomly generate a large number of evidence formulas $\alpha$, each time considering a large number of possible $\beta$s. If $\beta$ is MAP-entailed by $\alpha$, we add $\alpha \snake \beta$ to the set of positive examples; otherwise we add it to the set of negative examples. Notice that defaults in these experiments are not restricted to be just ``parallel'' rules.



We considered propositional MLNs learned from NLTCS, MSNBC, Plants and DNA data using the method from \cite{lowd}. These are standard datasets, and have 16, 17, 69 and 180 Boolean random variables, respectively. We used the existing train/tune/test division of the data. For each dataset, we generated 1000 training examples and 1000 testing examples of default rules as described above, and considered evidence formulas $\alpha$ of up 5 literals. We learn the possibilistic logic theory on the training examples, and report results on the held-out testing examples. To use the classical learners, we represent the antecedent using two Boolean attributes for each variable in the domain: the first indicates the variable's positive presence in the antecedent while the second indicates its negative presence. We represent the consequent in the same way. The label of an example is positive if it appears in the default theory and negative otherwise. While this allows us to predict whether a default $a \snake b$ should be true, the set of defaults predicted by the classical methods will in general not be consistent.

The last four rows of Table \ref{tab:mrfs} show the test set accuracy for each approach on each domain. Overall, the learned possibilistic logic theories have similar performance to the decision tree and random forest models, and outperform RIPPER. This is quite remarkable, as the possibilistic logic theories are much more interpretable (containing approximately 50\% fewer literals than the decision trees), 
which usually means that we have to accept a lower accuracy. Moreover, while the other methods can also be used for predicting MAP entailment, only our method results in a consistent logical theory, which could e.g.\ easily be combined with expert knowledge.


\section{Conclusions}
The aim of this paper was to study the problem of reasoning with default rules from a machine learning perspective. We have formally introduced the problem of learning from defaults and have analyzed its theoretical properties. Among others, we have shown that the complexity of the main decision problem is $\Sigma_2^P$ complete, and we have established asymptotically tight bounds on the VC-dimension. At the practical level, we have proposed practical heuristic learning algorithm, which can scale to datasets with thousands of rules. We have presented experimental results that show the application potential of the proposed learning algorithm, considering two different application settings: learning domain theories by crowdsourcing expert opinions and approximating propositional MLNs. We believe that the methods proposed in this paper will open the door to a wider range of applications of default reasoning, where we see defaults as a convenient interface between experts and learned domain models.

\section*{Acknowledgments}

This work has been supported by a grant from the Leverhulme Trust (RPG-2014-164). Jesse Davis is partially supported by the KU Leuven Research Fund  (C22/15/015), and FWO-Vlaanderen (G.0356.12, SBO-150033).

\ifdefined\ARXIVVERSION

\appendix

\section{Proofs of Lemmas}

\begin{proof}[Proof of Lemma \ref{lemma:orders}]
We prove this lemma by induction on $m$. The base case $m = 1$ is obvious. To show that the lemma holds for $m+1$ if it holds for $m$, let $\mathcal{C}_1', \mathcal{C}_2' \subseteq \mathcal{C}'$ and $\mathcal{C}_1, \mathcal{C}_2 \subseteq \mathcal{C}$ be restrictions of $\mathcal{C}'$ and $\mathcal{C}$ to inequalities only involving elements $\{x_1,\dots,x_{2^m} \}$ and $\{x_{2^m+1},\dots, x_{2^{m+1}} \}$, respectively. From the induction hypothesis we know that there is a permutation $\pi_1$ of $\{x_1,\dots,x_{2^m}\}$ and a permutation $\pi_2$ of $\{x_{2^m+1},\dots, x_{2^{m+1}} \}$ which satisfy the inequalities from $\mathcal{C}_1'$ and $\mathcal{C}_2'$, respectively, and no other inequalities from $\mathcal{C}_1\setminus \mathcal{C}_1'$ and $\mathcal{C}_2\setminus \mathcal{C}_2'$, respectively. Next we construct an auxiliary directed graph as follows. We create a directed path for both of the permutations $\pi_1$ and $\pi_2$ so that the $i$-th vertex in the path is labeled by the $i$-th element of the corresponding permutation. For every inequality $c =$ ``$\textit{kth}(\{x_1,\dots,x_{2^m}\}, i) < \textit{kth}(\{x_{2^m+1},\dots,x_{2^{m+1}}\}, i)$'', we add an edge $e_i$ between the $i$-th vertex of the first path and the $i$-th vertex of the second path. If $c \in \mathcal{C}'$ then we orient the edge $e_i$ in the direction from the second path to the first path, and in the other direction otherwise. It is easy to check that the resulting graph is acyclic. By topologically ordering the vertices of this graph, we obtain a permutation of $\mathcal{X}$ which satisfies all inequalities from $\mathcal{C}'$ and no inequalities from $\mathcal{C}\setminus \mathcal{C}'$.
\end{proof}

\begin{proof}[Proof of Lemma \ref{lemma:pos_order}]
($\Rightarrow$) If $(\mathcal{T}^*,  \textit{at-least}_{m_y-k+1}(\neg y_1,$ $\neg y_2,$ $\dots,\neg y_{m_y}) )\allowbreak\possent\allowbreak \textit{at-least}_{k}(z_1,\allowbreak z_2,\allowbreak\dots,\allowbreak z_{m_z})$ then  at least $k$ levels of the possibilistic logic theory $(\mathcal{T}^*,\allowbreak \{ \textit{at-least}_{m_y-k+1}(\neg y_1,\allowbreak\neg y_2,\dots,\neg y_{m_y}) \})$ of the form $\{ (z_i,\lambda_{z_i}) \}$, where $z_i \in \mathcal{Z}$, must be above the drowning level, and the number of non-drowned levels of the form $\{(y_i,\lambda_{y_i})\}$, where $y_i \in \mathcal{Y}$, must be equal to $k-1$, so that $m_y-k+1$ of $y_i$'s could be set to false. This also means that the $k$-th greatest element of $\mathcal{Z}$ must be greater than the $k$-th greatest element of $\mathcal{Y}$, which is what we needed to show. ($\Leftarrow$) If $\textit{kth}(\{y_1,\dots,y_{m_y}\}, k) < \textit{kth}(\{z_1,\dots,z_{m_z}\},k)$ then at least $k$ levels of $(\mathcal{T}^*, \{ \textit{at-least}_{m_y-k+1}(\neg y_1,\neg y_2,\dots,\neg y_{m_y}) \})$ of the form $\{ (z_i,\lambda_{z_i}) \}$, where $z_i \in \mathcal{Z}$, must be above the drowning level. Hence $(\mathcal{T}^*,\allowbreak \textit{at-least}_{m_y-k+1}(\neg y_1,\allowbreak\neg y_2,\allowbreak\dots,\allowbreak\neg y_{m_y}) ) \allowbreak\possent\allowbreak \textit{at-least}_{k}(z_1,z_2,\dots,z_{m_z})$ must be true.
\end{proof}

\section{Exact Parameter Learning for Separable Data}

For completeness, in this section we describe an algorithm which, given a theory $\mathcal{T}$ and a set of training examples $\mathcal{S}$, returns a stratification, called {\em separating stratification}, of $\mathcal{T}$ which covers all positive examples and no negative examples from $\mathcal{S}$, if one exists, and returns `NULL' otherwise. 

A naive algorithm for finding a separating stratification would simply enumerate all 
stratifications of $\mathcal{T}$ and check each time whether all positive and no negative examples are covered. It would stop and return the first such stratification found or return NULL at the end. This would need $O(|\mathcal{T}|^{|\mathcal{T}|} \cdot \log_2{|\mathcal{T}|} \cdot |\mathcal{S}^+ \cup \mathcal{S}^-|)$ queries to a SAT solver in the worst case, as there are $O(|\mathcal{T}|^{|\mathcal{T}|})$ possible stratifications to consider, and for each of these we need to verify $|\mathcal{S}^+ \cup \mathcal{S}^-|$ entailments of the form $(\Theta,\alpha)\possent \beta$, each of which requires $O(\log_2 |\mathcal{T}|)$ calls to a SAT solver. Algorithm \ref{alg:generic} outlines an exact approach that ``only'' needs $O(2^{|\mathcal{T}|} \cdot |\mathcal{S}^+ \cup \mathcal{S}^-|)$ calls to a SAT solver when combined with memoization-based dynamic programming. It searches through the possible stratifications by recursively stratifying their top levels.

\begin{algorithm}[tb]
\caption{\sc Stratify-Separable}
\label{alg:generic}
{\small
\begin{algorithmic}[1]
\Require{A propositional theory $\mathcal{T}$, a multiset of positive examples $\mathcal{S}^+$, a multiset of negative examples $\mathcal{S}^-$}
\Ensure{A stratification of $\mathcal{T}$ covering all positive and no negative examples. The stratification is represented as an ordered list of levels of the stratification where the lowest levels come first.}
\Statex
\State{{\bf global} $\textit{\textit{Closed}}$ $\leftarrow$ an empty hash set.}
\State{$\mathcal{S}^+_0 \leftarrow \{ \alpha \snake \beta \in \mathcal{S}^+ : \mathcal{T} \cup \{ \alpha \} \not\vdash \bot \mbox{ and } \mathcal{T} \cup \{ \alpha \} \vdash \beta \}$}
\State{{\bf return} stratify-impl($\mathcal{T}$, $\mathcal{S}^+\setminus \mathcal{S}^+_0$, $\mathcal{S}^-$)}
\Statex
\State{{\bf function} stratify-impl($\mathcal{T}$, $\mathcal{S}^+$, $\mathcal{S}^-$)}

\State{\quad{\bf if} $\textit{Closed}$ contains $\mathcal{T}$}
\State{\quad\quad} {\bf return} NULL
\State{\quad{\bf endif}}
\State{\quad{\bf if } $\mathcal{T} = \emptyset$ {\bf and} $\mathcal{S}^+ = \emptyset$ {\bf then}}
\State{\quad\quad{\bf return} $()$ /* i.e. an empty list */ }
\State{\quad{\bf endif}}
\State{\quad{\bf foreach} $\mathcal{T}' \subseteq \mathcal{T}$ such that (i) $\mathcal{T}' \cup \{\alpha \} \vdash \beta$
for all $\alpha \snake \beta \in \mathcal{S}^+$, and (ii) $\mathcal{T}' \cup \{ \alpha \} \vdash \bot$ or $\mathcal{T}' \cup \{ \alpha \} \not\vdash \beta$ for all $\alpha \snake \beta \in \mathcal{S}^-$}\label{line-for-loop}  
\State{\quad\quad $S^+_{\textit{cov}} \leftarrow \{ \alpha \snake \beta : \mathcal{T}' \cup \{ \alpha \} \not\vdash \bot \mbox{ and } \mathcal{T}' \cup \{ \alpha \} \vdash \beta \}$}
\State{\quad\quad $\textit{Stratification} \leftarrow $ stratify-impl($\mathcal{T}'$, $\mathcal{S}^+\setminus S^+_{\textit{cov}}$, $\mathcal{S}^-$)}
\State{\quad\quad{\bf if} $\textit{Stratification} \neq $ NULL {\bf then}}
\State{\quad\quad\quad {\bf return} $\textit{ConcatenateLists}((\mathcal{T} \setminus \mathcal{T}'), \textit{Stratification})$}
\State{\quad\quad{\bf endif}}
\State{\quad{\bf endforeach}}
\State{\quad {\bf store}($\textit{Closed}$, $\mathcal{T}$)}
\State{\quad{\bf return} NULL}
\State{{\bf end}}
\end{algorithmic}}
\end{algorithm}

\begin{table*}[t]
\begin{center}
  \begin{tabular}{l} 
\hline
$(\neg \textit{Fold} \vee \neg \textit{Raise},1)$, $(\neg \textit{Fold} \vee \neg \textit{Call},1)$, $(\neg \textit{Call} \vee \neg \textit{Raise},1),$ \\ \hline
$(\neg \textit{There\_are\_no\_straight\_or\_flush\_possibilities\_on\_the\_board}\vee \neg \textit{The\_opponent\_check-raised\_on\_the\_flop}\vee$ \\ \multicolumn{1}{r}{$\neg \textit{The\_opponent\_checked\_on\_the\_river}\vee \textit{call}, 0.75),$} \\
$(\neg \textit{The\_river\_card\_has\_just\_been\_dealt}\vee \neg \textit{You\_have\_three\_of\_a\_kind}\vee \neg \textit{There\_are\_three\_hearts\_on\_the\_board}\vee$ \\
\multicolumn{1}{r}{$\neg \textit{The\_opponent\_raised\_on\_the\_river}\vee \textit{Call}, 0.75),$} \\
$(\neg \textit{The\_river\_card\_has\_just\_been\_dealt}\vee \neg \textit{There\_are\_two\_pairs\_on\_the\_board}\vee$ \\
\multicolumn{1}{r}{$\neg \textit{The\_opponent\_checked\_the\_flop} \vee \textit{Call},0.75),$} \\
$(\neg\textit{The\_river\_card\_has\_just\_been\_dealt}\vee \neg \textit{The\_opponent\_checked\_the\_flop} \vee \neg \textit{The\_opponent\_check-raised\_on\_the\_turn}\vee$ \\
\multicolumn{1}{r}{$\neg \textit{The\_opponent\_raised\_on\_the\_river}, \textit{Call},0.75),$} \\
$(\neg\textit{The\_river\_card\_has\_just\_been\_dealt}\vee \neg\textit{You\_have\_three\_of\_a\_kind}\vee$ \\
\multicolumn{1}{r}{$\neg\textit{The\_opponent\_just\_called\_your\_bet\_on\_the\_turn}\vee \textit{Call},0.75)$,} \\ \hline
$(\neg \textit{You\_have\_a\_flush}\vee \textit{Raise}, 0.5)$, $(\neg \textit{You\_have\_a\_full\_house}\vee \textit{Raise}, 0.5)$, $(\neg \textit{You\_have\_three\_of\_a\_kind}\vee \textit{Raise}, 0.5)$,\\
$(\neg\textit{The\_opponent\_raised}\vee \neg \textit{The\_opponent\_has\_been\_playing\_very\_aggressive\_all\_evening}\vee \textit{Raise}, 0.5)$ \\
$(\neg\textit{There\_are\_no\_cards\_yet\_on\_the\_board}\vee \neg\textit{You\_have\_two\_3s}\vee \neg\textit{The\_opponent\_made\_a\_3-bet}\vee \textit{Fold},0.5)$, \\ \hline
$(\neg \textit{The\_flop\_cards\_have\_just\_been\_dealt} \vee \neg{You\_have\_a\_straight\_draw\_and\_a\_flush\_draw}, 0.25),$\\
$(\neg \textit{There\_are\_no\_cards\_yet\_on\_the\_board}\vee \neg \textit{You\_have\_2-3\_suited} \vee \textit{Raise}, 0.25)$, \\
$(\neg \textit{The\_opponent\_made\_a\_huge\_raise} \vee \textit{Raise},0.25),$ 
$(\textit{Call}, 0.25)$ \\ \hline
  \end{tabular}
  \end{center}
  \caption{Possibilistic logic theory learned in the crowd-sourced poker domain.}\label{tab:poker}
\end{table*}

\begin{theorem}\label{thm:correctness}
{\sc Stratify-Separable} finds a separating stratification if one exists, and returns NULL otherwise. When combined with memoization-based dynamic programming, it runs in time $O((4^{|\mathcal{T}|}+2^{|\mathcal{T}| + |\textit{vars}(\mathcal{T})|}) \cdot (|\mathcal{S}^+ \cup \mathcal{S}^-|))$ and uses $O(2^{|\mathcal{T}|} \cdot (|\mathcal{S}^+ \cup \mathcal{S}^-|))$ queries to an NP oracle, where $\textit{vars}(\mathcal{T})$ is the set of propositional variables in~$\mathcal{T}$.
\end{theorem}
\begin{proof} (Sketch)
 Algorithm \ref{alg:generic}'s correctness follows from the simple observation that a separating stratification exists if and only if $\mathcal{T}$ can be split into two sets, a non-empty set $\mathcal{T}'$ and a possibly empty set $\mathcal{T}''$ such that:
 \begin{itemize}
 \item there is a separating stratification ${\mathcal{T}''}^*$ of the set $\mathcal{T}''$ which covers positive examples from the set $\{ \alpha \snake \beta \in \mathcal{S}^+ : \mathcal{T} \cup \{ \alpha \} \vdash \bot \}$ and no negative examples from $\mathcal{S}^-$,
 \item $\{ (\alpha, 1) : \alpha \in \mathcal{T} \}$ does not cover any example from $\mathcal{S}^-$.
 \end{itemize}
This means that we can search through possible stratifications by refining their top level, which is what the recursion in Algorithm \ref{alg:generic} achieves. We can discard those stratifications which cover some of the negative examples. We can also discard stratifications for which we know that no refinement of the top level will cover all of the positive examples. These two conditions are checked on line \ref{line-for-loop} of Algorithm \ref{alg:generic}.

The upper bound on the number of SAT queries can be achieved simply by caching the results of queries. Finally, the total runtime bound can be achieved by using dynamic programming and memoization (i.e.\ by caching the results of the procedure $\textit{stratify-impl}$), and using the fact that SAT problems with $n$ variables can be solved in time $O(2^{n})$. Notice that $\textit{stratify-impl}$ will be called with at most $2^\mathcal{T}$ different sets of input arguments. This is because the second argument is actually always uniquely determined by the first argument $\mathcal{T}'$, given the conditions on $\mathcal{T}'$ in the for statement on line \ref{line-for-loop}: it is the set of positive examples $\alpha \snake \beta$ such that $\mathcal{T}' \cup \{\alpha \}$ is inconsistent. So, $\textit{stratify-impl}$ will be run at most $2^{|\mathcal{T}|}$ times and the for statement on line \ref{line-for-loop} will always go through at most $2^{|\mathcal{T}|}$ different sets $\mathcal{T}'$, while the cost of each iteration
will only be proportional to the number of examples, if the SAT queries are assumed to be pre-computed and cached. This gives us a worst-case bound on the runtime of the algorithm of $O((4^{|\mathcal{T}|}+2^{|\mathcal{T}| + |\textit{vars}(\mathcal{T})|}) \cdot (|\mathcal{S}^+| + |\mathcal{S}^-|))$.
\end{proof}

\begin{table}[t]
\begin{center}
  \begin{tabular}{l} \hline
	$(\neg a_{12} \vee a_{14}, 0.9),$ $(\neg a_{16} \vee a_{8}, 0.9), $ \\ \hline
	$(a_{8} \vee a_{15} \vee \neg a_{13} \vee a_{14}, 0.8), $ 	$(\neg a_{3} \vee a_{8}, 0.8), $ \\ 
	$(\neg a_{16} \vee a_{12}, 0.8), $ $(\neg a_{2} \vee a_{1}, 0.8), $ $(\neg a_{6} \vee a_{14}, 0.8), $ \\
	$(\neg a_{15} \vee a_{1}, 0.8), $ \\ \hline
	$(\neg a_{14} \vee a_{10}, 0.7), $ $(\neg a_{9}, 0.7), $ $(\neg a_{11} \vee a_{13}, 0.7), $ \\ \hline
	$(\neg a_{1} \vee a_{10}, 0.6), $ $(\neg a_{11} \vee a_{5}, 0.6), $ \\ \hline
	$(\neg a_{7}, 0.5), $ $(\neg a_{8}, 0.5), $ $(\neg a_{14} , 0.5), $ \\ \hline
	$(\neg a_{1}, 0.4), $ \\ \hline
	$(\neg a_{4}, 0.3), $ \\ \hline
	$(\neg a_{10}, 0.2), $ \\ \hline
	$(\neg a_{5}, 0.1).$
  \end{tabular}
  \end{center}
  \caption{Possibilistic logic theory learned for approximation of MAP inference in the NLTCS domain.}\label{tab:nltcs}
\end{table}

The implementation available online contains also an optimized version of the exact algorithm.Note that due to its high complexity the algorithm described in this section does not scale to problems involving large numbers of default rules. For practical problems, it is therefore preferable to use the heuristic algorithm described in Section~\ref{secHeuristicAlgorithm}.

\section{Learned Models}

In this section we briefly describe two examples of theories, one learned in the crowd-sourced poker domain (see Section \ref{sec:exp-crowd}) and the other learned in MAP-inference approximation experiments for the NLTCS domain (see Section \ref{sec:exp-map}).

The learned theory for the poker domain is shown in Table~\ref{tab:poker}. Since default rules in the dataset from which this theory was learned were all ``parallel rules'', most of the formulas in the theory are clausal representations of implications of the form ``if situation $\alpha$ then action $\beta$''; an exception to this is one of the rules in the lowest level which has the form ``not situation $\alpha$'', where $\alpha$ is in this case ``The flop cards have just been dealt  and you have a straight draw and a flush draw'', and this rule basically serves to block the other rules in this level for evidence $\alpha' \supseteq \alpha$. The top level of the theory consists of hard integrity constraints.

Table \ref{tab:nltcs} shows a small theory which was learned in the NLTCS domain after 20 iterations of the algorithm (the complete learned theory available online is larger). Since in this domain the default rules were not restricted to be of the ``parallel'' form, also the structure of the rules in the theory is more complex.

\fi

\bibliographystyle{named}
{\small
\bibliography{bibliography}
}

\end{document}